\documentclass{article}

\usepackage{arxiv}

\usepackage[utf8]{inputenc} 
\usepackage[T1]{fontenc}    
\usepackage{hyperref}       
\usepackage{url}            
\usepackage{booktabs}       
\usepackage{amsfonts}       
\usepackage{amsmath}
\usepackage{amsthm}
\usepackage{nicefrac}       
\usepackage{float}          
\usepackage{microtype}      
\usepackage{lipsum}
\usepackage{graphicx}
\usepackage[ruled]{algorithm2e} 
\graphicspath{ {./images/} {./figures/} }

\newtheorem{defin}{Definition}[section]

\newtheorem{proposition}{Proposition}[section]

\title{Profile-Then-Reason: Bounded Semantic Complexity for Tool-Augmented Language Agents}

\author{
 Paulo Akira F. Enabe \\
    Escola Politénica\\
    University of São Paulo\\
  \texttt{paulo.enabe@usp.br} \\
}

\begin{document}
\maketitle
\begin{abstract}
  Large language model agents that use external tools are often implemented through reactive execution, in which reasoning is repeatedly recomputed after each observation, increasing latency and sensitivity to error propagation. This work introduces Profile--Then--Reason (PTR), a bounded execution framework for structured tool-augmented reasoning, in which a language model first synthesizes an explicit workflow, deterministic or guarded operators execute that workflow, a verifier evaluates the resulting trace, and repair is invoked only when the original workflow is no longer reliable. A mathematical formulation is developed in which the full pipeline is expressed as a composition of profile, routing, execution, verification, repair, and reasoning operators; under bounded repair, the number of language-model calls is restricted to two in the nominal case and three in the worst case. Experiments against a ReAct baseline on six benchmarks and four language models show that PTR achieves the pairwise exact-match advantage in 16 of 24 configurations. The results indicate that PTR is particularly effective on retrieval-centered and decomposition-heavy tasks, whereas reactive execution remains preferable when success depends on substantial online adaptation.
\end{abstract}


\section{Introduction}

Large language models have emerged as a central paradigm in modern artificial intelligence due to their ability to learn broad linguistic, semantic, and procedural regularities from large-scale text corpora. Their recent development is rooted in the Transformer architecture introduced in \cite{vaswani2017attention}, in which sequence modeling is based on self-attention rather than recurrence or convolution. This architectural shift enabled substantially improved scalability in both training and inference, thereby making it possible to train models with very large parameter counts and broad task generalization capabilities. As a consequence, contemporary language models are able not only to generate coherent natural language, but also to perform instruction following, structured prediction, code generation, tool invocation, and multi-step reasoning. These capabilities have motivated the development of agentic frameworks in which a language model is no longer used solely as a text generator, but rather as a decision-making component embedded within a larger computational pipeline.

A major step in the development of reasoning-oriented prompting was the introduction of chain-of-thought prompting in \cite{wei2022chainofthought}. In that formulation, the model is no longer prompted to produce only an input-output mapping, but instead to generate a sequence of intermediate natural-language reasoning steps leading to the final answer. The significance of this contribution lies in showing that large language models can be induced, through prompting alone, to externalize part of their reasoning process in a form that is both operationally useful and partially interpretable. More precisely, chain-of-thought prompting demonstrated that the insertion of intermediate linguistic steps can improve performance on arithmetic, commonsense, and symbolic reasoning tasks, especially for models of sufficiently large scale \cite{wei2022chainofthought,wang2022selfconsistency}. In this sense, classical chain-of-thought prompting established the basic principle that reasoning quality can be improved by structuring the output space of the model, rather than by requesting only the final prediction. At the same time, the resulting reasoning process remains internal to a single language-model trajectory and does not, by itself, define an explicit execution framework for tool use, state transitions, or bounded interaction with external operators.

A subsequent line of work extended chain-of-thought reasoning beyond the generation of intermediate steps alone and toward explicit verification-oriented deliberation. In particular, the Chain-of-Verification framework proposed in \cite{dhuliawala2023chainofverification} considers a staged procedure in which the model first generates a draft response, then formulates verification questions, answers those questions independently, and finally revises its response in light of the obtained evidence. This development is closely related to chain-of-thought prompting in that it retains the central idea that intermediate reasoning structure can improve output quality. However, its objective is different. Rather than primarily eliciting latent reasoning ability for difficult tasks, it introduces an explicit self-checking mechanism aimed at reducing factual hallucination and improving response reliability. In this sense, Chain-of-Verification may be viewed as a verification-centered extension of the broader reasoning-by-decomposition paradigm initiated by classical chain-of-thought prompting.

A further development of particular importance for agentic systems is the ReAct framework introduced in \cite{yao2023react}. In contrast to classical chain-of-thought prompting, where the reasoning process remains internal to a single generative trajectory, ReAct interleaves natural-language reasoning traces with explicit task-specific actions executed in an external environment. More precisely, the model alternates between producing thoughts, which serve to decompose goals, track progress, or reformulate plans, and actions, which query an external source or modify the environment and thereby produce new observations \cite{yao2023react}. This coupling is methodologically significant because it allows reasoning to guide action selection while simultaneously allowing environmental feedback to update subsequent reasoning. In this sense, ReAct constitutes one of the first general prompt-based formulations in which large language models are used not only to explain intermediate reasoning steps, but also to support sequential decision making through explicit interaction with tools or environments. The empirical study in the paper demonstrated that such an interleaved reasoning-and-acting process can improve performance.

Despite its conceptual appeal and broad applicability, ReAct also exhibits structural limitations that become significant in tool-intensive settings. Because reasoning traces and external actions are interleaved sequentially, the framework typically requires a new language-model inference step after each observation, so that the overall computational cost and wall-clock latency grow with the length of the trajectory. This dependence on repeated model calls is not incidental, but intrinsic to the reactive formulation itself: each newly acquired observation may alter the subsequent reasoning state, which in turn necessitates further deliberation before the next action is selected \cite{yao2023react}. As a consequence, ReAct can be inefficient on tasks whose workflow is largely predictable in advance, since repeated semantic recomputation is performed even when intermediate observations only induce local parameter adjustments rather than genuine strategy changes. The same sequential dependence may also reduce robustness, as early retrieval errors, uninformative observations, or hallucinated intermediate thoughts can propagate through later steps of the trajectory. Indeed, the ReAct study reports error modes related to hallucination, reasoning error, and non-informative search. Related analyses have shown that these limitations are not confined to the original ReAct setting, but persist more broadly in reactive agent frameworks, particularly in the form of planning failures, error propagation, and efficiency overheads \cite{xing2024understanding, lu2025exploring, fan2025sweeffi, bogavelli2026agentarch}.

Two recent works are particularly relevant in the search for alternatives to fully reactive agent execution. ReWOO \cite{xu2023rewoo} proposes to decouple reasoning from observations by introducing a modular Planner--Worker--Solver architecture. The planner generates a complete sequence of interdependent plans before any tool feedback is incorporated, the worker executes the corresponding tool calls, and the solver combines the collected evidence into a final response. The principal motivation is to avoid the prompt redundancy and repeated language-model invocations induced by observation-dependent reasoning. By contrast, LLMCompiler \cite{kim2024llmcompiler} addresses the same general inefficiency problem from a systems perspective. Its formulation introduces a function-calling planner, a task-fetching unit, and an executor that collectively transform a user query into a dependency graph of tasks and execute compatible function calls in parallel. Thus, while both methods depart from the strictly sequential thought-action-observation pattern, ReWOO does so through prior decomposition of reasoning and evidence collection, whereas LLMCompiler does so through execution-graph optimization and parallel scheduling. These developments indicate that the design of efficient agent architectures remains an active area of research, particularly at the intersection of reasoning structure, execution orchestration, and empirical evaluation \cite{besta2024graph, sun2026dare}.

Although these developments substantially improve over fully reactive execution, an important methodological gap remains. Existing alternatives primarily address inefficiency either by decoupling reasoning from tool observations or by optimizing the orchestration of function calls. In many structured analytical settings, however, the principal difficulty is not only to reduce latency or token consumption, but to determine when a workflow can be reliably synthesized in advance, when deterministic execution is sufficient, and when bounded semantic adaptation becomes necessary. This distinction is particularly important in tool-intensive domains where the workflow class is regular enough to permit prior planning, yet still sensitive to schema irregularities, diagnostic contradictions, or execution failures. Therefore, what remains needed is a framework that preserves the computational advantages of precompiled execution while retaining a controlled mechanism for limited adaptation when the initial workflow ceases to be reliable.

The objective of this work is to introduce such a framework, referred to as Profile-Then-Reason. The proposed method is a bounded execution architecture for tool-augmented language agents in structured domains. More precisely, PTR separates semantic workflow synthesis from runtime execution and final interpretation. In the initial profile stage, the language model constructs an explicit workflow together with associated assumptions, execution annotations, and uncertainty descriptors. This workflow is then executed through deterministic or guarded operators, rather than through repeated observation-dependent reasoning. A subsequent verification stage evaluates whether the realized execution trace remains sufficiently trustworthy for final interpretation, and an additional repair stage is invoked only when execution evidence indicates that the original workflow is no longer reliable. In this sense, PTR is designed to reduce semantic recomputation relative to reactive architectures while preserving a bounded form of adaptivity when deterministic execution alone is insufficient.

The remainder of this paper is organized as follows. Section~\ref{sec:ptr-framework} introduces the Profile--Then--Reason framework and develops its mathematical formulation, including the profile, routing, execution, verification, repair, and reasoning operators, together with the associated boundedness and determinism properties. Section~\ref{sec:experiments} presents the experimental study, describing the benchmark suite, evaluation metrics, and comparative results against a ReAct baseline, with particular emphasis on the interaction between task structure and architectural performance. Finally, Section~\ref{sec:conclusion} concludes the paper by summarizing the main findings, discussing the methodological scope of the proposed framework, and outlining directions for future work.

\section{The Profile-Then-Reason Framework}
\label{sec:ptr-framework}

The objective of this work is to define a structured execution framework for tool-augmented language agents in domains where the admissible workflow class is sufficiently regular. The proposed framework, referred to as Profile-Then-Reason (PTR), is based on the observation that in many high-structure tasks the semantic burden of the agent can be concentrated at the beginning and at the end of execution. More precisely, the first language-model call is used to synthesize an executable workflow together with explicit uncertainty descriptors, whereas the final language-model call is used to interpret the resulting execution trace. Between these two stages, the workflow is executed by deterministic operators, possibly enriched by local branch rules, failure-recovery rules, and a post-execution verification step. In this sense, the formulation replaces repeated interleaved reasoning and acting by a bounded execution scheme whose computational cost, execution depth, and adaptation budget are explicitly controlled.

Let $\mathcal{X}$ denote the task space, let $\mathcal{M}$ denote the metadata space, let $\mathcal{S}$ denote the execution-state space, and let $\mathcal{Y}$ denote the report space. A task instance is denoted by $x \in \mathcal{X}$, and a metadata object is denoted by $m \in \mathcal{M}$. The metadata contains the structured information required for execution, including schema descriptors, tool schemas, domain constraints, and optional historical summaries. The complete PTR map may then be viewed as an operator
\begin{equation}
\mathfrak{P} : \mathcal{X} \times \mathcal{M} \to \mathcal{Y},
\end{equation}
constructed through the successive transformations
\begin{equation}
(x,m) \longmapsto p \longmapsto \rho \longmapsto s_f \longmapsto z_f \longmapsto y,
\end{equation}
where $p$ denotes a profile object, $\rho$ denotes an execution mode, $s_f$ denotes the final execution state, $z_f$ denotes a verification object, and $y$ denotes the final report. If the verification stage indicates that the realized execution path is not sufficiently trustworthy, an additional repair stage is inserted before the final reasoning stage. Therefore, the essential distinction of PTR lies not only in reducing the number of language-model calls, but in separating semantic workflow synthesis from the deterministic mechanics of execution.

At a high level, PTR consists of three semantic stages and three deterministic stages. The semantic stages are PROFILE, optional REPAIR, and REASON, each realized by a bounded language-model call. The deterministic stages are ROUTE, EXECUTE, and VERIFY. The framework thus separates semantic workflow synthesis from deterministic execution and verification, while preserving a finite adaptation budget: at most three language-model invocations and a bounded number of tool calls per task instance.

\subsection{Mathematical formulation}

To state the formulation precisely, let each task $x \in \mathcal{X}$ be represented by a tuple
\begin{equation}
x = (q,d,c),
\end{equation}
where $q$ is the task objective, $d$ is an optional data reference, and $c$ is an optional context object. Let each metadata object $m \in \mathcal{M}$ be represented by
\begin{equation}
m = (\Sigma,\Gamma,\Delta,H),
\end{equation}
where $\Sigma$ is a schema descriptor, $\Gamma$ is a finite catalog of available tools, $\Delta$ is a collection of domain constraints and execution policies, and $H$ is an optional historical summary used for calibration. Each tool $\tau \in \Gamma$ is modeled as a tuple
\begin{equation}
\tau = (\mathcal{U}_{\tau}, \mathcal{V}_{\tau}, F_{\tau}),
\end{equation}
where $\mathcal{U}_{\tau}$ is the parameter domain, $\mathcal{V}_{\tau}$ is the output space, and
\begin{equation}
F_{\tau} : \mathcal{U}_{\tau} \times \mathcal{S} \to \mathcal{V}_{\tau} \times \mathcal{S}
\end{equation}
is the corresponding state-transition operator. The dependence on the current state is made explicit because tool execution may consume prior observations, intermediate identifiers, or derived state variables produced earlier in the workflow.

A workflow is defined as a finite sequence of annotated tool invocations. More precisely, let
\begin{equation}
\mathcal{W} = \bigcup_{L \geq 1} \mathcal{W}_L,
\end{equation}
where $\mathcal{W}_L$ denotes the set of workflows of length $L$. An element $w \in \mathcal{W}_L$ is written as
\begin{equation}
w = \left( (\tau_{\ell}, u_{\ell}, \eta_{\ell}) \right)_{\ell=1}^{L},
\end{equation}
where $\tau_{\ell} \in \Gamma$ is a tool identifier, $u_{\ell} \in \mathcal{U}_{\tau_{\ell}}$ is a parameter object, and $\eta_{\ell}$ is an annotation object containing execution-relevant information such as auto-resolution markers, branch conditions, and recovery hints. The execution state is modeled as an element $s \in \mathcal{S}$ of the form
\begin{equation}
s = (r,e,b,f,\zeta),
\end{equation}
where $r$ is the result store, $e$ is the execution trace, $b$ is the branch-activation log, $f$ is the failure log, and $\zeta$ is an auxiliary environment object. The result store is a finite map containing all successful tool outputs and all derived execution quantities needed by later stages. The trace records the chronological sequence of tool calls and observations. The branch and failure logs record deterministic modifications and failed executions, respectively. This state representation is important because it makes the execution phase an explicit dynamical system on a finite structured state space.

The first language-model stage is the profile operator. Its role is to synthesize an executable workflow together with explicit descriptors of uncertainty and fragility. Formally, let
\begin{equation}
\Pi : \mathcal{X} \times \mathcal{M} \to \mathcal{P},
\end{equation}
where $\mathcal{P}$ is the space of admissible profile objects. A profile object is represented by
\begin{equation}
p = (w,\gamma,A,G,C,B,\Xi),
\end{equation}
where $w \in \mathcal{W}$ is the proposed workflow. The remaining components are partitioned into two groups. The epistemic descriptors, consisting of $\gamma \in [0,1]$ (plan confidence), $A$ (assumptions), and $G$ (fragile points), characterize the planner's uncertainty about the validity of the proposed workflow. The control descriptors, consisting of $C$ (replan conditions), $B$ (branch rules), and $\Xi$ (auxiliary annotations), encode the deterministic adaptation policy that governs execution. The mathematical role of $\Pi$ is to map a task and its structured context into a machine-usable execution policy. Its computational purpose is to concentrate semantic planning in a single synthesis step. An admissible profile must satisfy at least the following conditions: every referenced tool belongs to the catalog $\Gamma$, every parameter object is schema-compatible or marked for deterministic resolution, every branch rule is evaluable on the execution state, and every replan condition is expressible as a verifiable state predicate. These admissibility requirements are essential for computability, since without them the downstream execution operator would not be well posed.

Once a profile has been generated, the framework computes a deterministic risk score in order to select the execution mode. Let
\begin{equation}
R : \mathcal{M} \times \mathcal{P} \to [0,1]
\end{equation}
denote the risk functional. The quantity $R(m,p)$ is intended to summarize the degree to which the planned workflow is fragile with respect to the available metadata and the planner output. The formulation adopted in the present work uses a five-component decomposition. Let
\begin{equation}
c_1 = c_{\mathrm{schema}}(m,p), \qquad
c_2 = c_{\mathrm{planning}}(p), \qquad
c_3 = c_{\mathrm{method}}(m,p), \qquad
c_4 = c_{\mathrm{scale}}(m), \qquad
c_5 = c_{\mathrm{history}}(m),
\end{equation}
with $c_i \in [0,1]$ for $i=1,\dots,5$. The total risk is then defined as the convex combination
\begin{equation}
R(m,p) = \sum_{i=1}^{5} w_i c_i,
\end{equation}
subject to
\begin{equation}
w_i > 0, \qquad \sum_{i=1}^{5} w_i = 1.
\end{equation}
Thus,
\begin{equation}
0 \leq R(m,p) \leq 1.
\end{equation}
This construction is useful for both analysis and implementation. From the analytical viewpoint, it gives a bounded scalar functional with interpretable component contributions. From the computational viewpoint, it defines a computationally inexpensive and auditable routing criterion.

Given two thresholds $0 < \theta_1 < \theta_2 < 1$, let $\mathcal{R}_{\mathrm{mode}} = \{\mathrm{pure},\, \mathrm{guarded},\, \mathrm{repair\_eligible}\}$ denote the finite set of admissible execution modes. The routing map
\begin{equation}
\mathcal{Q} : \mathcal{M} \times \mathcal{P} \to \mathcal{R}_{\mathrm{mode}}
\end{equation}
is defined by
\begin{equation}
\mathcal{Q}(m,p) =
\begin{cases}
\mathrm{pure}, & \text{if } R(m,p) < \theta_1, \\
\mathrm{guarded}, & \text{if } \theta_1 \leq R(m,p) < \theta_2, \\
\mathrm{repair\_eligible}, & \text{if } R(m,p) \geq \theta_2.
\end{cases}
\end{equation}
The mode $\mathrm{pure}$ corresponds to direct deterministic execution. The mode $\mathrm{guarded}$ activates branch rules and local deterministic adaptations. The mode $\mathrm{repair\_eligible}$ activates the same guarded mechanism while also marking the run as structurally likely to require repair if the verifier later detects insufficient trust. It is worth mentioning that the router is entirely deterministic. Therefore, once the profile object is fixed, the routing decision is uniquely determined. This property is important for reproducibility and for post hoc analysis of execution behavior.

Given an initial state $s_0 \in \mathcal{S}$, a workflow $w \in \mathcal{W}$, and an execution mode $\rho \in \mathcal{R}_{\mathrm{mode}}$, the execution stage is defined by the operator
\begin{equation}
\mathcal{E} : \mathcal{W} \times \mathcal{R}_{\mathrm{mode}} \times \mathcal{S} \to \mathcal{S}.
\end{equation}
Let
\begin{equation}
w = \left( (\tau_{\ell}, u_{\ell}, \eta_{\ell}) \right)_{\ell=1}^{L}.
\end{equation}
The state is advanced recursively by
\begin{equation}
s_{\ell} = \mathcal{E}_{\ell}(s_{\ell-1};\tau_{\ell},u_{\ell},\eta_{\ell},\rho),
\qquad \ell = 1,\dots,L,
\end{equation}
with $s_f = s_L$. The local operator $\mathcal{E}_{\ell}$ is itself composed of a sequence of deterministic state-to-state transformations. First, parameters are resolved. Second, local branch rules are evaluated if the execution mode allows them. Third, the tool transition is executed. Fourth, if an error occurs, deterministic recovery rules are applied if available. Finally, the resulting output and all relevant execution metadata are written into the state. More precisely, let each component be a map $\Phi_{\ell}^{(\cdot)} : \mathcal{S} \to \mathcal{S}$, implicitly parameterized by the step data $(\tau_{\ell}, u_{\ell}, \eta_{\ell})$ and the execution mode $\rho$. The composition may then be written as
\begin{equation}
\mathcal{E}_{\ell}
=
\Phi_{\ell}^{\mathrm{store}}
\circ
\Phi_{\ell}^{\mathrm{recover}}
\circ
\Phi_{\ell}^{\mathrm{exec}}
\circ
\Phi_{\ell}^{\mathrm{branch}}
\circ
\Phi_{\ell}^{\mathrm{resolve}}.
\end{equation}
The operator $\Phi_{\ell}^{\mathrm{resolve}}$ eliminates unresolved symbolic parameters, $\Phi_{\ell}^{\mathrm{branch}}$ evaluates local branch predicates and applies parameter modifications when the execution mode permits, $\Phi_{\ell}^{\mathrm{exec}}$ invokes the tool transition $F_{\tau_{\ell}}$, $\Phi_{\ell}^{\mathrm{recover}}$ applies local retry rules when the tool call fails, and $\Phi_{\ell}^{\mathrm{store}}$ updates the state. In pure mode, $\Phi_{\ell}^{\mathrm{branch}}$ reduces to the identity. Therefore, the execution stage has the structure of a finite deterministic transition system once the profile object has been fixed.

The resolution stage deserves explicit discussion because it is one of the main mechanisms through which PTR removes unnecessary language-model calls. Let $\mathcal{A}$ denote the family of auto-resolution rules. Each rule is a map
\begin{equation}
a : \mathcal{S} \to \mathcal{V}_a,
\end{equation}
where $\mathcal{V}_a$ is the value space of the associated parameter slot. If a parameter is marked as \texttt{auto}, it is replaced by the output of the corresponding deterministic rule evaluated on the current state. Thus the resolved parameter object may be written as
\begin{equation}
u_{\ell}^{\mathrm{res}} = \Lambda_{\mathrm{auto}}(u_{\ell}, s_{\ell-1}).
\end{equation}
This construction is mathematically relevant because it guarantees that the actual tool invocation is always performed with a concrete admissible parameter object. It is computationally relevant because it moves local parameter selection from the language model to a low-overhead deterministic mechanism.

In guarded execution modes, the framework also permits branch rules. A branch rule is modeled as a pair
\begin{equation}
(\varphi,\psi),
\end{equation}
where
\begin{equation}
\varphi : \mathcal{S} \to \{0,1\}
\end{equation}
is a predicate and
\begin{equation}
\psi : \mathcal{U}_{\tau_{\ell}} \times \mathcal{S} \to \mathcal{U}_{\tau_{\ell}}
\end{equation}
is a parameter-modification map. If $\varphi(s_{\ell-1}) = 1$, then the pre-execution parameter object is updated according to
\begin{equation}
u_{\ell}^{\mathrm{br}} = \psi(u_{\ell}^{\mathrm{res}}, s_{\ell-1}),
\end{equation}
whereas otherwise $u_{\ell}^{\mathrm{br}} = u_{\ell}^{\mathrm{res}}$. The intended use of branch rules is local adaptation within a fixed method family. Thus, they are suitable for changes such as selecting a robust variant, widening a threshold, or switching between predetermined subprocedures, but they are not intended for conceptual redesign of the workflow. This distinction is crucial. Local variations remain in the deterministic layer, whereas global reinterpretation is reserved for the optional repair stage.

When a tool call fails, the framework attempts deterministic recovery before declaring structural failure. Let $\mathcal{E}_{\mathrm{err}}$ denote the error space. A recovery rule is a map
\begin{equation}
\chi : \mathcal{E}_{\mathrm{err}} \times \mathcal{U}_{\tau_{\ell}} \times \mathcal{S} \to \mathcal{U}_{\tau_{\ell}}.
\end{equation}
If a tool invocation fails with error $\varepsilon_{\ell}$ and an admissible recovery rule exists, the framework retries the same step with
\begin{equation}
u_{\ell}^{\mathrm{retry}} = \chi(\varepsilon_{\ell}, u_{\ell}^{\mathrm{br}}, s_{\ell-1}).
\end{equation}
A failure is called soft if the execution policy admits such a deterministic recovery or continuation. A failure is called hard if the workflow is no longer admissible under the structural constraints of the framework. This distinction is important because soft failures preserve the possibility of deterministic completion, whereas hard failures contribute directly to repair recommendation or to low trust in the final verification stage.

After the execution phase, the framework evaluates the structural quality of the realized trace by means of a verifier. The verifier does not attempt to reinterpret the semantics of the task. Instead, it assesses whether the produced trace is sufficiently informative and sufficiently stable to support the final reasoning stage. Let
\begin{equation}
\mathcal{V} : \mathcal{S} \times \mathcal{M} \times \mathcal{P} \to \mathcal{Z}
\end{equation}
denote the verifier operator, where $\mathcal{Z}$ is the space of verification objects. An element $z \in \mathcal{Z}$ is represented by
\begin{equation}
z = (\kappa,\sigma,I,\eta,\xi),
\end{equation}
where $\kappa \in [0,1]$ is the trust score, $\sigma$ is a status label, $I$ is a finite set of issues, $\eta$ is a finite set of reasoning flags, and $\xi \in \{0,1\}$ is the repair recommendation indicator. The framework does not prescribe a unique trust functional; any deterministic map $\kappa : \mathcal{S} \to [0,1]$ is admissible provided that it is monotonically non-increasing with respect to execution degradation. One concrete instance, used in the experiments of this work, takes the penalty form
\begin{equation}
\kappa(s)
=
\max \left\{
0,\,
1
-
\alpha_{\mathrm{fail}} n_{\mathrm{fail}}
-
\alpha_{\mathrm{empty}} n_{\mathrm{empty}}
-
\alpha_{\mathrm{thin}} n_{\mathrm{thin}}
-
\alpha_{\mathrm{branch}} n_{\mathrm{branch}}
-
\alpha_{\mathrm{diag}} \delta_{\mathrm{diag}}
\right\},
\end{equation}
where $n_{\mathrm{fail}}$, $n_{\mathrm{empty}}$, $n_{\mathrm{thin}}$, and $n_{\mathrm{branch}}$ are structural counters extracted from the execution trace, $\delta_{\mathrm{diag}}$ is an additional diagnostic penalty, and all coefficients are nonnegative and domain-dependent. The essential requirement is that $\kappa$ is a deterministic functional of the realized state, so that the verifier provides a reproducible mechanism for deciding whether the current execution path is adequate for interpretation.

Given a repair threshold $\theta_{\mathrm{rep}} \in (0,1)$, the repair recommendation indicator is defined by
\begin{equation}
\xi =
\begin{cases}
1, & \text{if } \kappa < \theta_{\mathrm{rep}} \text{ or a hard failure occurred}, \\
0, & \text{otherwise}.
\end{cases}
\end{equation}
If $\xi=0$, the framework proceeds directly to final reasoning. If $\xi=1$, an additional bounded repair stage is activated. The repair operator is defined by
\begin{equation}
\mathcal{K} : \mathcal{X} \times \mathcal{M} \times \mathcal{P} \times \mathcal{S} \times \mathcal{Z} \to \mathcal{P}_{\mathrm{patch}},
\end{equation}
where $\mathcal{P}_{\mathrm{patch}}$ is the space of admissible patched profiles. The mathematical role of $\mathcal{K}$ is not to produce a final answer, but to modify the workflow specification in light of the realized failures and diagnostics. In the present framework, repair is strictly bounded.

\begin{defin}[Bounded repair]
The PTR framework is said to satisfy bounded repair if at most one repair invocation is permitted for each task instance.
\end{defin}

\begin{proposition}[Bounded semantic complexity]
\label{prop:bounded-complexity}
Under bounded repair, if each workflow step admits at most $N_{\mathrm{rec}}$ deterministic retries, then for each task instance: (i) the total number of language-model invocations belongs to $\{2,3\}$, and (ii) the total number of tool invocations satisfies
\begin{equation}
N_{\mathrm{tool}} \leq (L + L^{\sharp})(1 + N_{\mathrm{rec}}),
\end{equation}
where $L$ is the length of the original workflow and $L^{\sharp}$ is the length of the repaired workflow, with $L^{\sharp} = 0$ in the absence of repair.
\end{proposition}

\begin{proof}
The framework always performs exactly one profile call and one reasoning call. By the bounded-repair condition, at most one additional repair call is permitted, giving at most three language-model invocations. The original and repaired workflows contribute at most $L$ and $L^{\sharp}$ steps, respectively, and each step may be attempted at most $1 + N_{\mathrm{rec}}$ times due to deterministic recovery.
\end{proof}

\begin{proposition}[Deterministic downstream execution]
\label{prop:deterministic-downstream}
Fix a task instance $(x,m)$ and an admissible profile object $p$. If the routing map $\mathcal{Q}$, the tool transitions $\{F_{\tau}\}_{\tau \in \Gamma}$, the auto-resolution rules $\mathcal{A}$, the branch rules $B$, the recovery rules in $\Delta$, and the verifier $\mathcal{V}$ are all deterministic, then every stage of PTR after profiling is deterministic.
\end{proposition}

\begin{proof}
The routing decision $\rho = \mathcal{Q}(m,p)$ is a deterministic function of $(m,p)$. Each execution step $\mathcal{E}_{\ell}$ is a composition of deterministic maps $\Phi_{\ell}^{\mathrm{resolve}}$, $\Phi_{\ell}^{\mathrm{branch}}$, $\Phi_{\ell}^{\mathrm{exec}}$, $\Phi_{\ell}^{\mathrm{recover}}$, and $\Phi_{\ell}^{\mathrm{store}}$, so the final execution state $s_f$ is uniquely determined by the initial state $s_0$ and the workflow $w$. The verification object $z = \mathcal{V}(s_f, m, p)$ is deterministic by assumption. Hence every downstream stage is uniquely determined once the profile $p$ is fixed.
\end{proof}

The final stage is the reasoning operator
\begin{equation}
\mathcal{G} : \mathcal{X} \times \mathcal{M} \times \mathcal{S} \times \mathcal{Z} \to \mathcal{Y}.
\end{equation}
Its purpose is to transform the final state and the verification object into a report that addresses the original task. The reasoning operator is constrained by the realized evidence. In particular, substantive claims must be grounded in the contents of the result store, and trust flags and diagnostic caveats must be propagated whenever relevant. Therefore, the final language-model call is not a new planning stage, but an interpretation stage acting on a fixed verified trace.

The complete PTR framework is thus defined by composition. First, the profile object is generated by $p=\Pi(x,m)$. Second, the route is computed by $\rho=\mathcal{Q}(m,p)$. Third, the workflow is executed to obtain a state $s=\mathcal{E}(w,\rho,s_0)$. Fourth, the verifier produces $z=\mathcal{V}(s,m,p)$. Fifth, if repair is recommended, a patched profile is obtained through $p^{\sharp}=\mathcal{K}(x,m,p,s,z)$, followed by re-execution and re-verification. Finally, the report is generated as
\begin{equation}
y = \mathcal{G}(x,m,s_f,z_f).
\end{equation}
This formulation is intrinsically discrete, since workflow length, retry counts, and repair budget are all finite. Nevertheless, the scalar risk and trust functionals admit a continuous analysis in the sense that they are mappings into compact intervals and can therefore be studied through standard monotonicity and sensitivity arguments.

Several additional structural properties follow directly from the construction. Proposition~\ref{prop:deterministic-downstream} establishes that all stages after profiling are deterministic under the stated assumptions, and Proposition~\ref{prop:bounded-complexity} guarantees finite execution depth. In addition, the risk functional is monotone with respect to each of its components, since
\begin{equation}
\frac{\partial R}{\partial c_i} = w_i > 0,
\end{equation}
for $i=1,\dots,5$. These properties do not by themselves imply semantic correctness, but they do establish computability, bounded execution depth, and an explicit separation between semantic synthesis and deterministic execution.

From an implementation standpoint, the semantic stages are isolated to the profile and reasoning operators, and optionally to the repair operator. All other components are deterministic and therefore computationally inexpensive, auditable, and straightforward to instrument. In this sense, the framework combines three desirable properties. It reduces the expected number of language-model calls relative to reactive execution, it improves observability because execution is represented explicitly as structured state transitions rather than as an implicit conversational trace, and it improves robustness because many local adaptations are handled by deterministic rules. The methodological scope of the framework is nevertheless limited. PTR is appropriate when the workflow family is sufficiently structured and when mid-execution adaptations can be encoded locally. It is not intended as a universal replacement for reactive agents in fully exploratory settings. Rather, the proposed method should be understood as a bounded execution framework for structured tool use.

\subsection{Implementation discussion}

The mathematical formulation presented above admits a direct algorithmic realization. Figure~\ref{fig:ptr-dataflow} illustrates the high-level information flow through the PTR pipeline, from the initial task and metadata inputs through the six phases to the final report. Algorithm~\ref{alg:ptr} then presents the complete execution procedure using the operators and notation established in the preceding sections. The following discussion traces the algorithm from input to output, emphasizing the correspondence between the formal operators and their procedural implementation and highlighting the design decisions that govern each phase.

\begin{figure}[H]
\centering
\includegraphics[width=0.25\textwidth]{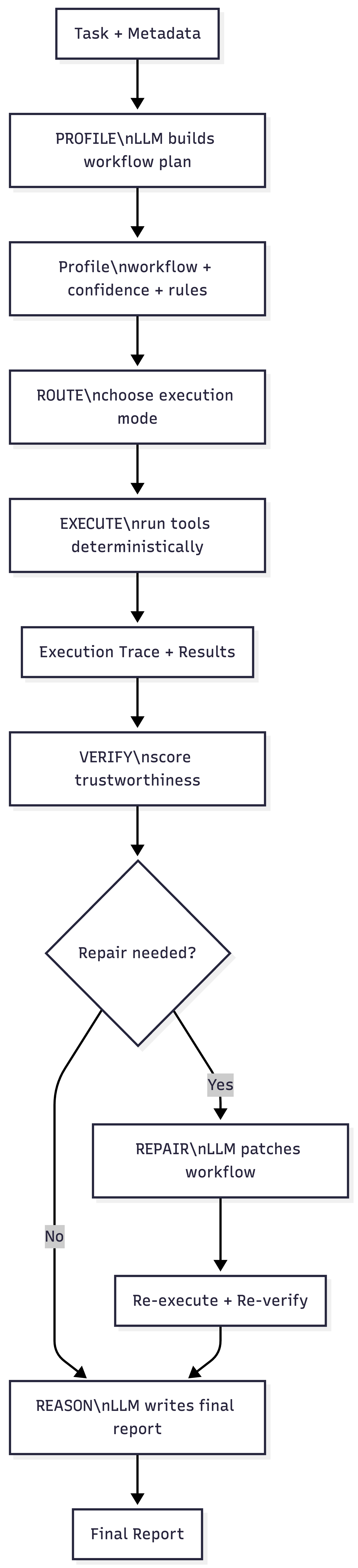}
\caption{Schematic representation of the PTR execution pipeline. Rectangular nodes denote state objects or operator applications, whereas the decision node denotes the deterministic repair trigger induced by the verification indicator $\xi$. The semantic stages PROFILE, REPAIR, and REASON are separated by deterministic routing, execution, and verification stages. The repair branch is activated at most once per task instance.}
\label{fig:ptr-dataflow}
\end{figure}

\begin{algorithm}[H]
\DontPrintSemicolon
\caption{Profile-Then-Reason (PTR) execution}\label{alg:ptr}
\KwIn{Task $x \in \mathcal{X}$,\; metadata $m \in \mathcal{M}$,\; tool catalog $\Gamma$,\; language model $\mathcal{L}$,\; rule set $\Delta$,\; budget $B$}
\KwOut{Report $y \in \mathcal{Y}$}
\BlankLine
\tcc{Phase 1: PROFILE (LLM call \#1)}
$p = (w,\gamma,A,G,C,B_r,\Xi) \leftarrow \Pi(x,m)$ via $\mathcal{L}$\;
$B.\mathrm{record}(\mathrm{cost}(p))$\;
\BlankLine
\tcc{Phase 1.5: ROUTE (deterministic)}
$R \leftarrow \sum_{i=1}^{5} w_i\, c_i(m,p)$\;
$\rho \leftarrow \mathcal{Q}(m,p)$ \tcp*{$\rho \in \mathcal{R}_{\mathrm{mode}}$}
\BlankLine
\tcc{Phase 2: EXECUTE (deterministic loop)}
$s \leftarrow s_0$\;
\For{$\ell = 1,\dots,|w|$}{
  $u_{\ell}^{\mathrm{res}} \leftarrow \Lambda_{\mathrm{auto}}(u_{\ell},\, s)$ \tcp*{resolve symbolic params}
  \uIf{$\rho \neq \mathrm{pure}$ \normalfont{and} $\exists\,(\varphi_j,\psi_j) \in B_r\!: \varphi_j(s)=1$}{
    $u_{\ell}^{\mathrm{br}} \leftarrow \psi_j(u_{\ell}^{\mathrm{res}},\, s)$ \tcp*{branch rule fires}
  }
  \lElse{$u_{\ell}^{\mathrm{br}} \leftarrow u_{\ell}^{\mathrm{res}}$}
  $(v_{\ell},\, s) \leftarrow F_{\tau_{\ell}}(u_{\ell}^{\mathrm{br}},\, s)$ \tcp*{tool transition}
  \If{$F_{\tau_{\ell}}$ failed \normalfont{and} $\exists\,\chi \in \Delta$ matching $\varepsilon_{\ell}$}{
    \For{$k = 1,\dots,N_{\mathrm{rec}}$}{
      $u_{\ell}^{\mathrm{retry}} \leftarrow \chi(\varepsilon_{\ell},\, u_{\ell}^{\mathrm{br}},\, s)$\;
      $(v_{\ell},\, s) \leftarrow F_{\tau_{\ell}}(u_{\ell}^{\mathrm{retry}},\, s)$\;
      \lIf{\normalfont{success}}{\textbf{break}}
    }
  }
}
$s_f \leftarrow s$\;
\BlankLine
\tcc{Phase 2.5: VERIFY (deterministic)}
$z = (\kappa,\sigma,I,\eta,\xi) \leftarrow \mathcal{V}(s_f,m,p)$\;
\BlankLine
\tcc{Phase 3: REPAIR (optional; at most one LLM call)}
\If{$\xi = 1$}{
  $p^{\sharp} \leftarrow \mathcal{K}(x,m,p,s_f,z)$ via $\mathcal{L}$\;
  $B.\mathrm{record}(\mathrm{cost}(p^{\sharp}))$\;
  Re-execute Phase 2 with workflow $w^{\sharp}$ from $p^{\sharp}$; update $s_f$\;
  $z \leftarrow \mathcal{V}(s_f,m,p^{\sharp})$ \tcp*{re-verify}
}
\BlankLine
\tcc{Final Phase: REASON (LLM call \#2 or \#3)}
$y \leftarrow \mathcal{G}(x,m,s_f,z)$ via $\mathcal{L}$\;
$B.\mathrm{record}(\mathrm{cost}(y))$\;
\Return{$y$}
\end{algorithm}

The procedure begins with the profile phase, shown in the first block of Algorithm~\ref{alg:ptr}. Given a task $x$ and metadata $m$, the profile operator $\Pi$ is realized by a single call to the language model $\mathcal{L}$. The prompt supplied to $\mathcal{L}$ is constructed from the task description, the schema descriptor $\Sigma$, the tool catalog $\Gamma$, the domain constraints $\Delta$, and, when available, the historical summary $H$. The language model is instructed to return a structured object encoding the full profile $p = (w, \gamma, A, G, C, B_r, \Xi)$. This output is then parsed and validated: every tool reference must belong to $\Gamma$, every parameter must be schema-compatible or marked for auto-resolution, and every branch rule must be expressible in the grammar accepted by the runtime evaluator. If parsing fails, one retry with an error-correcting prompt is permitted before the run is aborted. The cost of this call is recorded in the budget $B$, and execution proceeds only if the budget constraint is not violated. It is worth noting that the profile phase is the only stage in which the language model performs unconstrained generation. All subsequent deterministic stages either execute mechanically or, in the case of repair and reasoning, invoke the language model within a tightly scoped prompt.

Once the profile is available, the routing phase evaluates the five risk components $c_1, \dots, c_5$ and computes the aggregate risk score $R(m,p)$ as a convex combination with the weight vector $\{w_i\}_{i=1}^{5}$, as shown in the second block of Algorithm~\ref{alg:ptr}. The routing map $\mathcal{Q}$ then assigns an execution mode $\rho \in \mathcal{R}_{\mathrm{mode}}$ by comparing $R$ against the thresholds $\theta_1$ and $\theta_2$. This computation is entirely deterministic: it involves five bounded evaluations, one dot product, and two comparisons, and therefore contributes negligible overhead. The significance of the routing decision is that it governs which adaptation mechanisms are active during execution. In pure mode, the execution loop runs without branch-rule evaluation, producing the lightest and most predictable path. In guarded mode, branch predicates are evaluated before each tool call, permitting local parameter adjustments within a fixed method family. In repair-eligible mode, the same guarded mechanism is active, but the run is additionally marked as structurally likely to benefit from repair if the verifier later detects degraded trust. Since the router is deterministic, the execution mode is uniquely determined once the profile object has been fixed, which is important for reproducibility and post hoc analysis.

The execution phase, shown in the central loop of Algorithm~\ref{alg:ptr}, realizes the operator composition $\mathcal{E}_{\ell} = \Phi_{\ell}^{\mathrm{store}} \circ \Phi_{\ell}^{\mathrm{recover}} \circ \Phi_{\ell}^{\mathrm{exec}} \circ \Phi_{\ell}^{\mathrm{branch}} \circ \Phi_{\ell}^{\mathrm{resolve}}$ as a sequential loop over the planned workflow steps. For each step $\ell$, the implementation first applies the auto-resolution operator $\Lambda_{\mathrm{auto}}$ to replace any symbolic parameters marked as \texttt{auto} with concrete values derived from the current state $s$. If the execution mode is not pure and a branch predicate $\varphi_j$ from the control descriptor set $B_r$ evaluates to $1$ on the current state, the corresponding parameter-modification map $\psi_j$ is applied, yielding the branched parameter object $u_{\ell}^{\mathrm{br}}$; otherwise the resolved parameters pass through unchanged. The tool transition $F_{\tau_{\ell}}$ is then invoked with $u_{\ell}^{\mathrm{br}}$ and the current state, producing both a tool output $v_{\ell}$ and an updated state. If the tool call fails with an error $\varepsilon_{\ell}$ and a matching recovery rule $\chi$ exists in the rule set $\Delta$, the framework retries the invocation with modified parameters up to $N_{\mathrm{rec}}$ times. Outputs are accumulated in the result store component of the state, with keys deduplicated by appending a sequential suffix when the same tool is invoked more than once. The entire execution loop is deterministic once the profile has been fixed: no language-model call is made between tool invocations. Mid-execution adaptation is confined to auto-resolution rules, which map prior results to concrete parameter values, and branch rules, which conditionally modify parameters based on evaluable state predicates. Both mechanisms are either pre-programmed in $\Delta$ or proposed by the profile operator and evaluated by a fixed grammar-based engine; they are therefore auditable, testable, and contribute no additional language-model cost.

After the execution loop terminates, the verification phase applies the verifier $\mathcal{V}$ to the final state $s_f$, the metadata $m$, and the profile $p$, as indicated in Algorithm~\ref{alg:ptr}. The verifier examines the execution trace for structural indicators of degradation --- failed steps, empty or thin tool outputs, recovery activations, and branch-rule firings --- and produces the verification object $z = (\kappa, \sigma, I, \eta, \xi)$. The trust score $\kappa$ is computed as a deterministic functional of the trace counters and penalty coefficients, and the repair recommendation indicator $\xi$ is set to $1$ if $\kappa$ falls below the repair threshold $\theta_{\mathrm{rep}}$ or if a hard failure was recorded. The verifier does not attempt to interpret the substantive content of tool outputs; it only assesses whether the execution path is structurally adequate to support the final reasoning stage. This separation is deliberate: trust estimation should not depend on, or be influenced by, narrative generation.

If the verification phase recommends repair, the framework invokes the repair operator $\mathcal{K}$ through a second language-model call, as shown in the conditional block of Algorithm~\ref{alg:ptr}. The prompt for this call includes the original task, the original profile, the execution trace, and the verification diagnostics. The language model is instructed to return a patched profile $p^{\sharp}$ containing a modified workflow and updated parameters, but is explicitly prohibited from generating a final answer. The patched workflow is then executed through the same deterministic loop described above, and the resulting state is re-verified. Crucially, the framework permits at most one repair cycle: if the re-verified trace still exhibits low trust, the system proceeds to the reasoning phase with the available evidence and appropriate caveats, rather than entering an iterative replanning loop. This boundedness guarantee is formalized in Proposition~\ref{prop:bounded-complexity}.

The procedure concludes with the reasoning phase, shown in the final block of Algorithm~\ref{alg:ptr}. The reasoning operator $\mathcal{G}$ is realized through the last language-model call, whose prompt contains the original task, the metadata, the final execution state $s_f$, and the verification object $z$ including any trust flags and diagnostic caveats. The language model is instructed to synthesize a report grounded exclusively in the realized evidence: substantive claims must reference tool outputs contained in the result store, trust caveats must be propagated whenever the verification status indicates reduced confidence, and no statistics or conclusions may be fabricated beyond what the execution trace supports. The final language-model call is therefore not a new planning stage, but an interpretation stage acting on a fixed verified trace.

Throughout the procedure, every language-model invocation is followed by a budget check, ensuring that the total cost does not exceed the prescribed limit $B$. The overall cost profile is thus predictable: most runs incur exactly two language-model calls, corresponding to the profile operator $\Pi$ and the reasoning operator $\mathcal{G}$, while runs that trigger repair incur a third call for $\mathcal{K}$. The deterministic phases --- routing, execution, and verification --- contribute negligible computational overhead relative to the language-model invocations, since they involve only rule evaluation, dictionary lookups, and tool calls.

The implementation described above makes explicit the structural difference between PTR and reactive agent architectures such as ReAct. In reactive frameworks, the language model is invoked after every tool observation to decide the next action, so the number of language-model calls scales linearly with the number of tool steps and is not bounded a priori. In PTR, the language model is invoked at most three times per task instance, regardless of workflow length. All adaptation between tool calls is handled by deterministic operators: auto-resolution rules replace symbolic parameters with concrete values derived from prior results, and branch rules conditionally modify parameters or skip steps based on evaluable state predicates. These mechanisms are evaluated by a fixed grammar-based engine rather than by open-ended language-model generation, and they are therefore auditable, testable, and free of additional language-model cost. The language model is re-engaged only when the verifier determines that the execution trace is structurally inadequate, and even then at most once. This separation of semantic synthesis from deterministic execution is the principal architectural distinction of the PTR framework.

\section{Experimental results}
\label{sec:experiments}

The purpose of the present section is to examine whether the structural properties of the Profile--Then--Reason framework established in the previous section yield corresponding empirical advantages in practice. The methodology proposed in this work is motivated by three principal considerations, namely bounded semantic complexity, deterministic or guarded execution, and selective bounded adaptation through repair. These properties suggest that PTR should be capable of reducing the computational burden associated with reactive agent trajectories, while preserving or improving answer quality in domains whose workflow class is sufficiently regular. However, such a claim cannot be justified on structural grounds alone. It must be evaluated against competitive baselines across tasks exhibiting different degrees of decomposability, retrieval dependence, and sensitivity to intermediate observations.

For this reason, the experimental study is designed not only to measure aggregate performance, but also to identify the regimes in which the proposed execution model is well aligned with the structure of the underlying task. In particular, the experiments seek to determine whether prior workflow synthesis is advantageous when the tool sequence is short, regular, and largely foreseeable, and conversely whether reactive execution remains preferable when intermediate observations fundamentally alter the reasoning trajectory. Thus, the objective is twofold. First, the study provides an empirical validation of the theoretical motivation for PTR by assessing its behavior in terms of accuracy, robustness, latency, and cost. Second, it serves to delimit the methodological scope of the framework by identifying the classes of problems for which bounded execution is an appropriate design principle and those for which open-ended observation-dependent reasoning retains a practical advantage.

The results reported below should therefore be interpreted as an analysis of architectural behavior rather than as a leaderboard-style comparison in isolation. More precisely, the comparison with ReAct is intended to reveal how different execution principles interact with task structure, model class, and tool dependence. In this sense, the experimental section complements the formal development of the previous section by testing whether the bounded execution hypothesis underlying PTR is reflected in measurable improvements at benchmark scale, and by clarifying the conditions under which such improvements can or cannot be expected.

\subsection{Methodological setup}

The objective of the experimental study is to assess whether the structural properties of PTR described in the previous section translate into measurable gains in accuracy, robustness, and cost efficiency relative to a standard ReAct baseline. The comparison is performed across six benchmarks chosen to span qualitatively distinct task regimes:

\begin{itemize}
\item \textbf{TriviaQA}~\cite{joshi2017triviaqa} is a reading comprehension benchmark containing over 650,000 question-answer-evidence triples authored by trivia enthusiasts. Questions feature complex, compositional phrasing that requires mapping creative lexical variations to the correct Wikipedia entity. Each question admits multiple valid answer aliases (e.g., ``NYC'', ``New York City'', ``New York''), and evaluation selects the best match across all accepted forms.

\item \textbf{NQ-Open}~\cite{kwiatkowski2019natural} consists of real queries issued to the Google search engine, paired with short answers extracted from Wikipedia articles. Unlike the polished trivia phrasing of TriviaQA, NQ-Open questions are colloquial and often ambiguous (e.g., ``when was the last time anyone was on the moon''), making the dataset a realistic test of open-domain question answering under real-world query distributions.

\item \textbf{StrategyQA}~\cite{geva2021aristotle} contains yes/no questions that require implicit multi-step reasoning. Each question demands decomposition into subquestions, information retrieval for each subquestion, and compositional reasoning over the results (e.g., ``Did Aristotle use a laptop?'' requires knowing when Aristotle lived and when laptops were invented). The dataset provides gold decomposition annotations, though agents must discover these steps independently.

\item \textbf{GSM8K}~\cite{cobbe2021training} is a dataset of approximately 8,800 linguistically diverse grade-school math word problems, each requiring two to eight steps of elementary arithmetic. Each solution includes natural-language chain-of-thought reasoning with a final numeric answer. Problems involve basic operations applied to real-world scenarios, testing multi-step mathematical reasoning rather than information retrieval.

\item \textbf{AQuA-RAT}~\cite{amini2019mathqa} comprises approximately 100,000 algebraic word problems presented in a five-way multiple-choice format (A--E), accompanied by natural-language rationales. Problems span algebra, probability, geometry, and number theory, requiring equation setup and symbolic manipulation beyond the elementary arithmetic of GSM8K. The multiple-choice format provides a $20\%$ random baseline but also allows models to verify answers against the given options.

\item \textbf{HotPotQA}~\cite{yang2018hotpotqa} is a multi-hop question answering dataset where each question requires reasoning over two or more Wikipedia articles. Questions are constructed so that the answer cannot be found in a single document: the agent must retrieve one article, extract a bridging entity, and then retrieve a second article to find the answer (e.g., ``What university did the director of Inception attend?'' requires chaining from the film to the director to their education).
\end{itemize}

For each dataset, $50$ questions were evaluated under four language models, namely GPT-4o-Mini, Claude Haiku 4.5, GPT-5.4, and Claude Sonnet 4.6, yielding $24$ model-dataset configurations and $48$ total agent runs. Both PTR and ReAct were evaluated with the same tool set, consisting of a Wikipedia search tool and a Wikipedia lookup tool, under deterministic decoding with temperature $0.0$. The ReAct baseline was allowed up to eight thought-action-observation iterations, whereas PTR was restricted to the bounded PROFILE $\rightarrow$ ROUTE $\rightarrow$ EXECUTE $\rightarrow$ VERIFY $\rightarrow$ (REPAIR) $\rightarrow$ REASON pipeline described in Section~\ref{sec:ptr-framework}.

\subsection{Evaluation metrics}

The primary accuracy metric is exact match (EM). For a given question with predicted answer $\hat{a}$ and gold answer $a^{*}$, the exact-match indicator is defined as
\begin{equation}
\mathrm{EM}(\hat{a}, a^{*}) =
\begin{cases}
1, & \text{if } \mathrm{norm}(\hat{a}) = \mathrm{norm}(a^{*}), \\
0, & \text{otherwise},
\end{cases}
\end{equation}
where $\mathrm{norm}(\cdot)$ denotes a benchmark-specific normalization function. For free-text benchmarks (HotPotQA, TriviaQA, NQ-Open), normalization consists of lowercasing, removal of articles, punctuation, and redundant whitespace. For TriviaQA and NQ-Open, the evaluation is alias-aware: the prediction is compared against all accepted answer strings, and a match with any alias is counted as correct. For StrategyQA, the answer is extracted as a binary yes/no label via pattern matching. For GSM8K, the answer is extracted as a numeric value with appropriate normalization (e.g., $18.0 = 18$). For AQuA-RAT, the answer is extracted as a single letter from A to E. The dataset-level exact-match score is the arithmetic mean of the per-question indicators over the evaluation set.

As a secondary accuracy metric, the token-level F1 score is reported for free-text benchmarks (HotPotQA, TriviaQA, NQ-Open). Let $P$ and $G$ denote the multisets of word tokens in the normalized prediction and gold answer, respectively. The token-level precision, recall, and F1 are defined as
\begin{equation}
\mathrm{prec} = \frac{|P \cap G|}{|P|}, \qquad
\mathrm{rec} = \frac{|P \cap G|}{|G|}, \qquad
\mathrm{F1} = \frac{2 \cdot \mathrm{prec} \cdot \mathrm{rec}}{\mathrm{prec} + \mathrm{rec}},
\end{equation}
with the convention $\mathrm{F1} = 0$ when both precision and recall are zero. For benchmarks with binary, numeric, or multiple-choice evaluation (StrategyQA, GSM8K, AQuA-RAT), the F1 score coincides with exact match and is therefore not reported separately.

Efficiency is measured along four dimensions: the total number of language-model calls per evaluation run, the total number of input and output tokens consumed, the total cost in US dollars computed from provider-specific pricing, and the mean wall-clock latency per question.

For pairwise comparison between frameworks, three derived quantities are used. The exact-match difference $\Delta\mathrm{EM} = \mathrm{EM}_{\mathrm{PTR}} - \mathrm{EM}_{\mathrm{ReAct}}$ measures the accuracy gap, with positive values indicating a PTR advantage. The cost ratio $\mathrm{EM}_{\mathrm{ReAct\text{-}cost}} / \mathrm{EM}_{\mathrm{PTR\text{-}cost}}$ measures relative efficiency, with values greater than one indicating that PTR is less expensive. For a given model-dataset configuration, the framework that achieves the strictly higher exact-match score is said to hold the \textbf{pairwise advantage} for that configuration. This criterion is used throughout as the primary unit of comparison.

Several methodological limitations should be noted. Each configuration was evaluated in a single run with a fixed random seed ($\mathrm{seed} = 42$) and no repeated trials. Consequently, no variance estimates or confidence intervals are reported. With $50$ questions per configuration, exact-match differences below approximately $0.10$ may not be statistically significant under standard binomial tests. The present analysis therefore relies on the consistency of observed patterns across multiple benchmarks and models, rather than on per-configuration significance tests.

\subsection{Results}

Table~\ref{tab:overall-results} reports the exact-match results for all model-dataset pairs. At the most aggregate level, PTR holds the pairwise advantage in $16$ of the $24$ configurations ($67\%$). This result indicates that the proposed bounded execution strategy is not merely competitive in isolated settings, but is a viable default configuration across a heterogeneous benchmark suite. However, the results also show that the superiority of one framework over the other is not uniform across task classes. Rather, the dominant explanatory variable is the task structure itself. In particular, PTR exhibits a clear advantage on retrieval-centered tasks, whereas ReAct retains an advantage on tasks that require iterative search refinement or unconstrained symbolic reasoning. 

\begin{table}[H]
\centering
\caption{Exact-match results for PTR and ReAct across all datasets and models. The column Adv.\ indicates which framework holds the pairwise advantage for each configuration.}
\label{tab:overall-results}
\begin{tabular}{llcccc}
\toprule
Dataset & Model & PTR & ReAct & Adv. \\
\midrule
TriviaQA & GPT-4o-Mini & 0.660 & 0.160 & PTR \\
TriviaQA & Claude Haiku 4.5 & 0.760 & 0.580 & PTR \\
TriviaQA & GPT-5.4 & 0.780 & 0.700 & PTR \\
TriviaQA & Claude Sonnet 4.6 & 0.820 & 0.520 & PTR \\
\midrule
NQ-Open & GPT-4o-Mini & 0.260 & 0.060 & PTR \\
NQ-Open & Claude Haiku 4.5 & 0.300 & 0.020 & PTR \\
NQ-Open & GPT-5.4 & 0.469 & 0.160 & PTR \\
NQ-Open & Claude Sonnet 4.6 & 0.380 & 0.000 & PTR \\
\midrule
StrategyQA & GPT-4o-Mini & 0.580 & 0.360 & PTR \\
StrategyQA & Claude Haiku 4.5 & 0.720 & 0.780 & ReAct \\
StrategyQA & GPT-5.4 & 0.820 & 0.780 & PTR \\
StrategyQA & Claude Sonnet 4.6 & 0.800 & 0.720 & PTR \\
\midrule
GSM8K & GPT-4o-Mini & 0.837 & 0.380 & PTR \\
GSM8K & Claude Haiku 4.5 & 0.960 & 0.900 & PTR \\
GSM8K & GPT-5.4 & 0.660 & 0.860 & ReAct \\
GSM8K & Claude Sonnet 4.6 & 0.980 & 0.780 & PTR \\
\midrule
AQuA-RAT & GPT-4o-Mini & 0.180 & 0.061 & PTR \\
AQuA-RAT & Claude Haiku 4.5 & 0.860 & 0.880 & ReAct \\
AQuA-RAT & GPT-5.4 & 0.320 & 0.780 & ReAct \\
AQuA-RAT & Claude Sonnet 4.6 & 0.780 & 0.900 & ReAct \\
\midrule
HotPotQA & GPT-4o-Mini & 0.120 & 0.080 & PTR \\
HotPotQA & Claude Haiku 4.5 & 0.020 & 0.265 & ReAct \\
HotPotQA & GPT-5.4 & 0.122 & 0.360 & ReAct \\
HotPotQA & Claude Sonnet 4.6 & 0.080 & 0.167 & ReAct \\
\bottomrule
\end{tabular}
\end{table}

A clearer picture emerges when the results are grouped by task family. On factual retrieval benchmarks, PTR holds the pairwise advantage consistently. On TriviaQA, PTR outperforms ReAct in all four model configurations, with exact-match gains ranging from $0.08$ to $0.50$ and an average advantage of $0.265$. On NQ-Open, PTR again holds the advantage in all four configurations, with an average advantage of $0.292$. These two datasets are precisely those for which the assumptions underlying PTR are most strongly satisfied: the task can usually be decomposed into a compact search plan, the relevant tool sequence is short, and intermediate observations do not typically require conceptual replanning. In such cases, the PROFILE stage is sufficient to identify a viable retrieval strategy, and the deterministic execution layer prevents the repeated reformulation loops that are frequently observed under ReAct. It is therefore reasonable to interpret the superiority of PTR on these datasets as evidence that bounded planning is particularly well matched to tool-dependent factoid retrieval. 

StrategyQA occupies an intermediate position. This dataset requires implicit decomposition rather than simple one-shot retrieval, since a yes-or-no answer often depends on several latent subquestions. Nevertheless, PTR holds the pairwise advantage in three of the four model configurations, with an average exact-match advantage of $0.070$. This result is noteworthy because it shows that PTR is not limited to trivial lookup tasks. Rather, the PROFILE stage is often capable of generating a sufficiently accurate decomposition in advance, provided that the resulting substeps remain structurally simple. In this sense, the results on StrategyQA support the claim that the planning advantage of PTR extends beyond single-hop retrieval, at least when the decomposition can be expressed as a short finite workflow. 

The results on GSM8K are more subtle. PTR holds the pairwise advantage in three of the four model configurations and exhibits an average exact-match advantage of $0.129$. At first sight, this may appear surprising, since GSM8K is primarily a reasoning benchmark rather than a retrieval benchmark. However, the observed behavior is consistent with the interpretation of PTR as a bounded semantic scaffold. In the stronger GSM8K configurations, the REASON stage is able to solve the arithmetic task even when the intermediate tool interactions are of limited utility. Thus, the PROFILE and VERIFY stages do not necessarily contribute to mathematical reasoning directly, but they do not prevent successful solution in most cases. It is worth mentioning, however, that this gain in accuracy is accompanied by a cost penalty: on GSM8K, PTR is more expensive than ReAct because the tool-planning overhead is not compensated by any substantial retrieval benefit. Therefore, the results on GSM8K should be interpreted as evidence of robustness rather than of architectural optimality for pure arithmetic tasks. 

In contrast, AQuA-RAT and HotPotQA reveal the principal limitations of PTR. On AQuA-RAT, ReAct holds the pairwise advantage in three of the four model configurations, with an average advantage of $0.120$. This behavior is consistent with the hypothesis that structured JSON planning may constrain models whose performance depends on freer symbolic manipulation. AQuA-RAT contains algebraic multiple-choice problems in which the dominant challenge is not tool selection but the maintenance of a coherent reasoning trajectory over a sequence of symbolic steps. In that regime, the bounded and highly structured PROFILE representation appears to interfere with rather than support the strongest models. On HotPotQA, ReAct again holds the advantage in three of the four configurations, with an average advantage of $0.120$. Here the explanation is different. HotPotQA requires genuine multi-hop retrieval, in which the correct second query often depends on an entity discovered only after the first query has been executed. This task violates one of the central premises of PTR, namely that the workflow can be usefully planned before execution begins. Since ReAct can revise its search strategy after each observation, it is naturally better suited to this benchmark. These two datasets therefore delimit the methodological scope of the proposed framework: PTR is effective when the workflow is structurally compressible, but not when the solution path requires substantial mid-execution semantic adaptation. 

The model-wise analysis further reinforces the importance of task structure. GPT-4o-Mini is the most striking case: PTR holds the pairwise advantage on all six benchmarks for this model. This suggests that the structured planning scaffold provided by PTR can compensate, to a considerable extent, for limited standalone tool-use ability. In contrast, GPT-5.4 and Claude Haiku 4.5 exhibit an even split, with three pairwise advantages each for PTR and ReAct, and Claude Sonnet 4.6 shows four advantages for PTR and two for ReAct. These results indicate that model capability alone does not determine the preferred agent architecture. Rather, the interaction between model class and task structure is decisive. On retrieval-centered tasks, even weaker models benefit strongly from the bounded PTR scaffold, whereas on algebraic or iterative multi-hop tasks stronger models can exploit the flexibility of ReAct more effectively. Thus, task type appears to be a more reliable predictor of the advantaged architecture than model size or nominal capability.

The cost analysis is also informative. Across all experiments, PTR incurred a total cost of \$6.29, compared with \$7.12 for ReAct, corresponding to an overall reduction of approximately $12\%$. This reduction is not uniform across datasets. PTR is cheaper on NQ-Open, StrategyQA, and HotPotQA, whereas ReAct is cheaper on TriviaQA, GSM8K, and AQuA-RAT. The most important point is not the sign of the difference on every individual dataset, but the fact that PTR achieves lower total cost while also holding the pairwise advantage in a majority of the configurations. This observation is consistent with the theoretical boundedness established in Section~\ref{sec:ptr-framework}. Since PTR restricts the number of language-model calls to two in the common case and to three when repair is triggered, its semantic complexity is controlled a priori. By contrast, the number of ReAct iterations depends on the realized trajectory and may increase without a corresponding increase in answer quality. Therefore, the cost results provide empirical support for the claim that bounded semantic complexity is not merely a formal property, but has practical consequences at benchmark scale. 
To summarize the benchmark evidence more compactly, Table~\ref{tab:dataset-summary} and Figure~\ref{fig:delta-em} group the results by dataset and report the corresponding average exact-match differences. Both representations make clear that PTR is the superior default choice on tool-dependent factual retrieval and on implicit decomposition tasks, whereas ReAct remains preferable when success depends on open-ended symbolic reasoning or on search trajectories that must be refined online.

\begin{table}[H]
\centering
\caption{Dataset-level comparison between PTR and ReAct. Pairwise advantage counts indicate the number of model configurations in which each framework achieves the strictly higher exact-match score. The average advantage is computed as PTR EM minus ReAct EM, averaged over the four models.}
\label{tab:dataset-summary}
\begin{tabular}{lccc}
\toprule
Dataset & PTR adv. & ReAct adv. & Avg. $\Delta$EM \\
\midrule
TriviaQA & 4/4 & 0/4 & +0.265 \\
NQ-Open & 4/4 & 0/4 & +0.292 \\
StrategyQA & 3/4 & 1/4 & +0.070 \\
GSM8K & 3/4 & 1/4 & +0.129 \\
AQuA-RAT & 1/4 & 3/4 & -0.120 \\
HotPotQA & 1/4 & 3/4 & -0.120 \\
\bottomrule
\end{tabular}
\end{table}

\begin{figure}[t]
\centering
\includegraphics[width=0.92\linewidth]{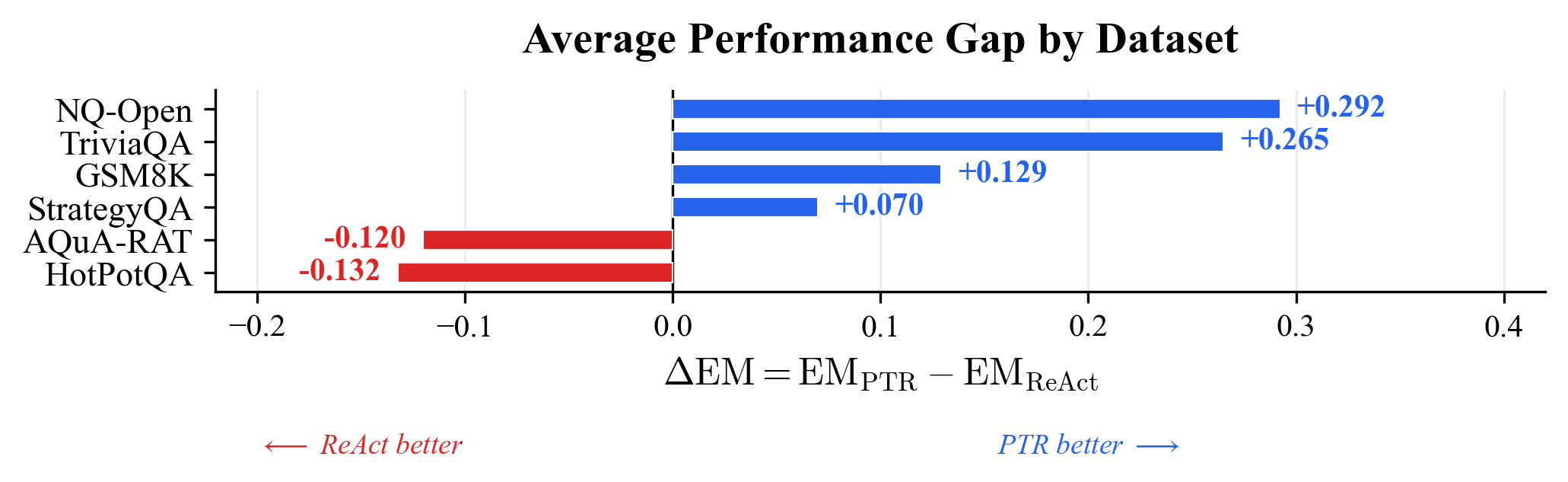}
\caption{Dataset-level average exact-match difference $\Delta \mathrm{EM} = \mathrm{EM}_{\mathrm{PTR}} - \mathrm{EM}_{\mathrm{ReAct}}$, averaged over the four evaluated language models. Positive values indicate a PTR advantage, whereas negative values indicate a ReAct advantage. PTR is favored on retrieval-centered and decomposition-heavy tasks (NQ-Open, TriviaQA, GSM8K, StrategyQA), while ReAct retains an advantage on AQuA-RAT and HotPotQA, where success depends more strongly on symbolic flexibility or online search refinement.}
\label{fig:delta-em}
\end{figure}

Several conclusions follow from these experiments. First, the results support the central design premise of PTR: when the admissible workflow class is sufficiently regular, it is advantageous to synthesize the workflow once, execute it deterministically, and reserve semantic interpretation for the final stage. Second, the benefit of this design is task-dependent rather than universal. The framework is particularly effective for retrieval-dominant problems and for decomposition tasks that can be expressed as a short finite plan. Third, the experiments also expose a meaningful limitation: when later reasoning steps must be conditioned on previously unseen intermediate observations, a reactive architecture retains an advantage. This is consistent with the formal structure of PTR, which deliberately trades open-ended mid-execution adaptation for bounded semantic complexity. Therefore, the empirical evidence should not be interpreted as showing that PTR supersedes ReAct in all settings. Rather, it shows that PTR is a strong default architecture over a broad and practically important subset of tool-augmented tasks. 

A final observation concerns the implications for future architecture design. The results suggest that the choice between PTR and ReAct should itself be treated as a routing problem. In particular, the empirical evidence indicates that retrieval-centered tasks and many decomposition tasks should be directed to PTR, whereas hard algebraic reasoning and genuine multi-hop retrieval should be directed to a more reactive policy. In this sense, the present results motivate not only the use of PTR as a standalone framework, but also the design of higher-level meta-controllers capable of selecting between bounded and reactive execution regimes on the basis of task structure. 

\section{Conclusion}
\label{sec:conclusion}

This work introduced Profile--Then--Reason, a bounded execution framework for tool-augmented language agents in structured domains. The proposed formulation was motivated by the observation that, in many high-structure tasks, repeated observation-dependent reasoning is not necessary at every execution step. Instead, the semantic burden of the agent can often be concentrated in an initial workflow-synthesis stage and a final interpretation stage, while the intermediate execution process is delegated to deterministic or guarded operators. In this sense, PTR was designed to separate semantic planning from runtime mechanics, to bound the number of language-model invocations a priori, and to retain a controlled form of adaptivity through verification and single-step repair.

From a methodological standpoint, the principal contribution of the framework is the explicit decomposition of agent execution into the PROFILE, ROUTE, EXECUTE, VERIFY, optional REPAIR, and REASON stages. This decomposition makes it possible to state the execution process as a composition of well-defined operators on structured task, metadata, workflow, and state spaces. The resulting formulation is computationally attractive because it admits bounded semantic complexity, deterministic downstream execution once the profile has been fixed, and a transparent separation between local deterministic adaptation and global semantic replanning. These properties do not by themselves guarantee semantic correctness, but they do establish computability, auditability, bounded execution depth, and a principled basis for implementation.

The empirical results support the central design hypothesis of the framework. Across the evaluated benchmarks, PTR showed clear advantages on retrieval-centered and decomposition-heavy tasks, where the admissible workflow class is sufficiently regular to permit useful prior planning. In such settings, the bounded execution strategy frequently yielded higher exact-match performance together with lower aggregate cost than the ReAct baseline. At the same time, the experiments also clarified the limits of the approach. On tasks such as HotPotQA and AQuA-RAT, where success depends more strongly on online search refinement or unconstrained symbolic manipulation, the flexibility of reactive execution remained advantageous. Therefore, the results should not be interpreted as establishing PTR as a universal replacement for reactive agent architectures. Rather, they indicate that PTR is a strong default framework for a broad and practically important class of structured tool-augmented tasks.

Several directions for future work follow naturally from the present study. First, the current repair mechanism is deliberately bounded to a single semantic intervention. It would be of interest to investigate more refined repair policies that remain analytically controlled while improving robustness on borderline cases. Second, the risk and trust functionals used for routing and verification are currently hand-designed and deterministic. A more systematic study of their calibration, sensitivity, and domain transfer properties would strengthen the framework further. Third, the experimental evidence suggests that the choice between PTR and ReAct should itself be treated as a higher-level routing problem. This points toward hybrid architectures in which a meta-controller selects between bounded and reactive execution regimes on the basis of task structure and anticipated workflow regularity. Finally, it would be worthwhile to evaluate the framework in more domain-specific environments, particularly those involving structured data analysis, scientific workflows, and enterprise tool ecosystems, where prior workflow synthesis and deterministic execution are especially natural.

In summary, the present work argues that the design of language-agent architectures should not be framed solely as a choice between static prompting and fully reactive execution. There exists an intermediate regime in which workflows can be synthesized semantically, executed deterministically, and repaired only when execution evidence justifies doing so. PTR is an explicit formulation of this regime. Its theoretical boundedness properties and empirical behavior suggest that such architectures deserve further study as a principled alternative for structured tool-augmented reasoning.

\newpage

\bibliographystyle{unsrt}  


\newpage



\end{document}